\documentclass[letterpaper, 10 pt, journal, twoside]{IEEEtran}
\DeclareUnicodeCharacter{FFFD}{ }

\ifCLASSINFOpdf
\else
\fi

\usepackage{graphicx} 
\usepackage{epsfig} 
\usepackage{dblfloatfix} 
\usepackage{courier}
\usepackage{amsmath} 
\usepackage{amssymb}  
\usepackage{amsthm}
\usepackage{xcolor}
\usepackage{amsmath} 
\usepackage{amssymb}
\usepackage{float}
\usepackage{algorithm}		
\usepackage{algorithmicx} 	
\usepackage{algpseudocode}  
\usepackage{todonotes}
\usepackage{multicol}
\usepackage[backend=bibtex,bibstyle=ieee,citestyle=numeric-comp,doi=false,isbn=false,url=false,eprint=false]{biblatex}
\newcommand{\ignore}[1]{}

\newcommand{\norm}[1]{\left\Vert#1\right\Vert} 
\newcommand{\mc}[1]{\mathcal{#1}}

\newcommand{\bma}[1]{\left[\begin{array}{ #1}}
\newcommand{\ema}{\end{array}\right]}

\DeclareMathAlphabet{\mbf}{OT1}{ptm}{b}{n}
\newcommand{\mbs}[1]{{\boldsymbol{#1}}}

\newcommand{\mbfdot}[1]{{\dot{\mbf{#1}}}}
\newcommand{\mbfbar}[1]{{\bar{\mbf{#1}}}}
\newcommand{\mbfhat}[1]{{\hat{\mbf{#1}}}}
\newcommand{\mbfcheck}[1]{{\check{\mbf{#1}}}}

\newcommand{\mbftilde}[1]{{\tilde{\mbf{#1}}}}

 \newcommand{\rframe}[1]{{\ensuremath {\mathcal{F}}_{#1}}}

\def\fdota{{\raisebox{-2pt}{\LARGE $\cdot$}}}
\def\fdotb{{\raisebox{-0.6ex}{ \kern0.2ex\raisebox{0.8ex}{\tiny $\hspace*{-1ex}\circ$}}}}
\def\fddota{{\raisebox{-2pt}{\LARGE $\cdot\hspace*{-0.2ex}\cdot$}}}
\def\fddotb{{\raisebox{-0.6ex}{ \kern0.2ex\raisebox{0.8ex}{\tiny $\hspace*{-1ex}\circ\circ$}}}}

\newtheorem{theorem}{Theorem}[section] 

\newcommand{\p}{\partial}

\newcommand{\trans}{{\ensuremath{\mathsf{T}}}} 
\newcommand{\utimes}{ {\raisebox{-0.6ex}{ \kern-1.0ex\raisebox{0.6ex}{ \small $\mathsf{v}$}}} } %
\newcommand{\trace}{ {\ensuremath{\mathrm{tr}}} } 

\newcommand{\beq}{\begin{equation}}
\newcommand{\eeq}{\end{equation}}
\newcommand{\bdis}{\begin{displaymath}}
\newcommand{\edis}{\end{displaymath}}
\newcommand{\beqarray}{\begin{eqnarray}}
\newcommand{\eeqarray}{\end{eqnarray}}
\newcommand{\beqarraynn}{\begin{eqnarray*}}
\newcommand{\eeqarraynn}{\end{eqnarray*}}
\newcommand{\balign}{\begin{align}}
\newcommand{\ealign}{\end{align}}
\newcommand{\balignnn}{\begin{align*}}
\newcommand{\ealignnn}{\end{align}}

\renewcommand{\theenumii}{\arabic{enumii}}
\makeatletter
\renewcommand{\p@enumii}{\theenumi.}
\makeatother

\begin{document}
%
%
%
%
%
%
%
\def \myJournal {IEEE Robotics and Automation Letters}
\def \myDoi {10.1109/LRA.2021.3067640}
\def \myPaperSiteName {IEEE Xplore}
\def \myPaperSiteLink {https://ieeexplore.ieee.org/document/9382085}
\def \myYear {2021}
\def \myPaperCitation{C. C. Cossette, M. Shalaby, D. Saussié, J. R. Forbes and J. Le Ny, ``Relative Position Estimation Between Two UWB Devices With IMUs,'' in \textit{IEEE Robotics and Automation Letters}, vol. 6, no. 3, pp. 4313-4320, July 2021.}


\begin{figure*}[t]

\thispagestyle{empty}
\begin{center}
\begin{minipage}{6in}
\centering
This paper has been accepted for publication in \emph{\myJournal}. 
\vspace{1em}

This is the author's version of an article that has, or will be, published in this journal or conference. Changes were, or will be, made to this version by the publisher prior to publication.
\vspace{2em}

\begin{tabular}{rl}
DOI: & \myDoi\\
\myPaperSiteName: & \texttt{\myPaperSiteLink}
\end{tabular}

\vspace{2em}
Please cite this paper as:

\myPaperCitation

\vspace{15cm}
\copyright \myYear \hspace{4pt}IEEE. Personal use of this material is permitted. Permission from IEEE must be obtained for all other uses, in any current or future media, including reprinting/republishing this material for advertising or promotional purposes, creating new collective works, for resale or redistribution to servers or lists, or reuse of any copyrighted component of this work in other works.

\end{minipage}
\end{center}
\end{figure*}
\newpage
\clearpage
\pagenumbering{arabic} 
\title{Relative Position Estimation Between Two UWB Devices with IMUs}
%
%
%

\author{Charles Champagne Cossette$^1$, Mohammed Shalaby$^1$,  David Saussi\'e$^2$, James Richard Forbes$^1$, \\ Jerome Le Ny$^2$%
\thanks{Manuscript received: Oct. 15th, 2020; Revised Jan. 15th, 2021; Accepted Feb. 14th, 2021}
\thanks{This paper was recommended for publication by Editor Javier Civera upon evaluation of the Associate Editor and Reviewers' comments.
*This work was supported by the FRQNT under grant 2018-PR-253646, the JELF, the William Dawson Scholar Program, and the NSERC Discovery Grant Program.} 
\thanks{$^1$C. C. Cossette, M. Shalaby, J.R. Forbes are with the Department of Mech. Engineering, McGill University. { \small charles.cossette@mail.mcgill.ca, mohammed.shalaby@mail.mcgill.ca, james.richard.forbes@mcgill.ca.}}%
\thanks{$^2$D. Saussi\'e, J. Le Ny, are with the Department of Electrical Engineering, Polytechnique Montr\'eal. 
{\small d.saussie@polymtl.ca}, {\small jerome.le-ny@polymtl.ca}.}%
\thanks{Digital Object Identifier (DOI): see top of this page.}
}
%
%

\markboth{IEEE Robotics and Automation Letters. Preprint Version. Accepted February, 2021}
{Cossette \MakeLowercase{\textit{et al.}}: Relative Position Estimation Between Two UWB Devices with IMUs} 

%



\maketitle

\begin{abstract}
    For a team of robots to work collaboratively, it is crucial that each robot have the ability to determine the position of their neighbors, relative to themselves, in order to execute tasks autonomously. This letter presents an algorithm for determining the three-dimensional relative position between two mobile robots, each using nothing more than a single ultra-wideband transceiver, an accelerometer, a rate gyro, and a magnetometer. A sliding window filter estimates the relative position at selected keypoints by combining the distance measurements with acceleration estimates, which each agent computes using an on-board attitude estimator. The algorithm is appropriate for real-time implementation, and has been tested in simulation and experiment, where it comfortably outperforms standard estimators. A positioning accuracy of less than 1 meter is achieved with inexpensive sensors.
\end{abstract}

\begin{IEEEkeywords}
    Localization; Range Sensing; Multi-Robot Systems
\end{IEEEkeywords}

%
\IEEEpeerreviewmaketitle

\section{Introduction}
%
%
%
%
\IEEEPARstart{T}{he} ability to determine the relative position between two robots, or \emph{agents}, is needed in many applications. For instance, multi-robot tasks such as collaborative exploration and mapping, search and rescue, and formation control, each require relative position information. Access to inter-robot distance measurements is becoming increasingly accessible due to technologies such as ultra-wideband radio (UWB). UWB in particular is attractive due to its low cost, low power, high accuracy, and ability to function in GPS-denied environments. Such advantages have even warranted the presence of UWB transceivers in smartphones and smartwatches, which could be used to provide position information of one user relative to another or, more generally, any other UWB transceiver. However, determining a full three-dimensional relative position estimate from a single distance measurement is impossible. There is an infinite set of relative positions that will produce the same distance measurement. Hence, to realize an observable relative positioning solution, more information is required. 
\begin{figure}[t]
    \centering
    \includegraphics[width = \linewidth, clip = true, trim = {0cm 0cm 0cm 10cm}]{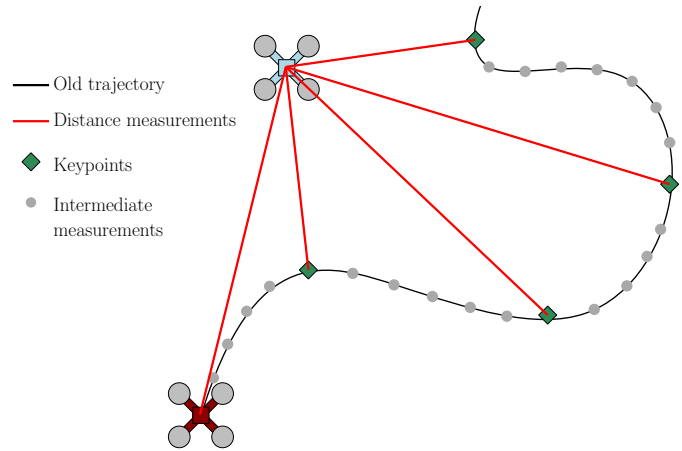}
    \caption{An example scenario where two quadrotors possess UWB modules, providing distance measurements between them. Under sufficient motion, it is possible to determine the relative position between them in a common reference frame.}
\end{figure}

Distance-based positioning is very commonly achieved by measuring distances to several landmarks with known positions, referred to as \emph{anchors} when using UWB \cite{Ledergerber2015,Cano2019a}. However, the problem of estimating the relative position between moving robots, without any external reference, is much more difficult. { Nevertheless, there are several papers that achieve relative positioning from distance measurements, and do so by using additional sensors.
For example, cameras are often used in addition to UWB \cite{Xu2020}, but vision requires substantial computational capabilities, which imposes a lower bound on the size of the agent. In \cite{Shalaby2021}, the presence of some agents with two UWB tags makes the relative positions observable.
The observability analysis provided by \cite{Batista2011} concludes that when velocity information is available in addition to single-distance measurements, the relative position is observable provided that a \emph{persistency of excitation} condition is satisfied, meaning that the agents must be in continuous relative motion. This is simulated in \cite{Sarras2018}, where velocity measurements are assumed to be available, and experimented in \cite{Guo2017a}, however GPS is used to provide displacement information for the agents. 
Cameras with optical flow are used to measure agent motion in \cite{VanderHelm2018, Nguyen2019,Nguyen2020}, and \cite{Liu2017a} proposes a particle filter to estimate the relative positions of a group of mobile UWB devices, which also have IMUs. However, their solution requires a central server to perform the estimation task. The literature lacks three-dimensional solutions that use only UWB and IMU measurements.}
%
%
%

The main contribution of this letter is a three-dimensional relative position estimation algorithm for two mobile agents, each using nothing more than a single sensor measuring the distance between them, and a 9-DOF IMU. No fixed infrastructure, GPS, cameras, or heavy computing is required, and the position is resolved in a known local frame, such as an East-North-Up reference frame. The proposed solution is easily decentralizable, and achieves sub-meter level accuracy in experiment. The algorithm requires persistency of excitation, meaning that the agents must be moving in a non-planar trajectory.

The remainder of this letter is as follows. Section \ref{sec:preliminaries} introduces the notation and probabilistic estimation tools used in this letter. Section \ref{sec:algorithm} outlines the proposed algorithm, with simulation and experimental results presented in Sections \ref{sec:simulation} and \ref{sec:experiment}, respectively.

\section{Preliminaries} \label{sec:preliminaries}
\subsection{Notation}
In this letter, $p(\mbf{x})$ denotes the joint probability density function (PDF) of the continuous random variable $\mbf{x} \in \mathbb{R}^n$. If $\mbf{x}$ is normally distributed with mean $\mbs{\mu}$ and covariance $\mbs{\Sigma}$, it is written as $\mbf{x} \sim \mc{N}(\mbs{\mu}, \mbs{\Sigma})$. The compressed notation $\mbf{x}_{0:K} = [\mbf{x}_0^\trans \ldots \mbf{x}_K^\trans ]^\trans$ is used for brevity. A bolded $\mbf{1}$ indicates an appropriately-sized identity matrix. 

An arbitrary reference frame `$a$' is denoted $\rframe{a}$. Physical vectors resolved in $\rframe{a}$ are denoted with a subscript $\mbf{v}_a$. The same physical vector resolved in a different frame $\rframe{b}$ can be related by a direction cosine matrix (rotation matrix), denoted $\mbf{C}_{ab} \in SO(3)$, such that $\mbf{v}_a = \mbf{C}_{ab}\mbf{v}_b$ \cite{Farrell2008}. A direction cosine matrix (DCM) $\mbf{C}_{ab}$ can be parameterized using a rotation vector, denoted $\mbs{\phi}\in \mathbb{R}^3$, using  $\mbf{C}_{ab} = \exp\left(\mbs{\phi}^{\times}\right)$, where $(\cdot)^\times$ denotes the skew-symmetric cross-product matrix operator. 
In the context of this letter, $\rframe{b}$ is associated with an agent's body frame, and $\rframe{a}$ refers to some local common frame, such as an East-North-Up frame.

\subsection{Maximum A Posteriori Estimation}
Consider generic, nonlinear process and measurement models, given respectively by 
\begin{align} \label{eq:map3}
    \mbf{x}_k &= \mbf{f}(\mbf{x}_{k-1}, \mbf{u}_{k-1}) + \mbf{w}_{k-1},  & k= 1,\ldots,K, \\
    \mbf{y}_k &= \mbf{g}(\mbf{x}_k) + \mbf{v}_k,  &k = 0,\ldots, K, \label{eq:map4}
\end{align}
where $\mbf{x}_k \in \mathbb{R}^n$ denotes the system state at time step $k$, $\mbf{u}_{k-1}$ is the system input, and $\mbf{y}_k$ is the output. The terms $\mbf{w}_{k-1} \sim \mc{N}(\mbf{0}, \mbf{Q}_{k-1})$ and $\mbf{v}_k \sim \mc{N}(\mbf{0}, \mbf{R}_k)$ represent random process and measurement noise, respectively. 

The \emph{maximum a posteriori} framework aims to find the states $\mbf{x}_{0:K}$ which maximize the posterior probability density function, given a history of input measurements $\mbf{u}_{0:K-1}$, output measurements $\mbf{y}_{0:K}$, and a prior distribution of the initial state, $\mbf{x}_0 \sim \mc{N}(\mbfcheck{x}_0, \mbfcheck{P}_0)$. That is, the estimate $\mbfhat{x}_{0:K}$ is given by
\begin{align} \label{eq:map1}
    \mbfhat{x} = \arg \max_{\mbf{x}} p (\mbf{x} | \mbfcheck{x}_0, \mbf{u}_{0:K-1}, \mbf{y}_{0:K}),
\end{align}
where $\mbf{x} = \mbf{x}_{0:K}$ has been written without subscripts to reduce notation. {It is well known that \eqref{eq:map1} is equivalently posed as a nonlinear least-squares problem \cite[Ch.~3.1]{Barfoot2019} given by}
\beq \label{eq:map2}
\mbfhat{x} = \arg\min_{\mbf{x}} \frac{1}{2} \mbf{e}(\mbf{x})^\trans \mbf{W} \mbf{e}(\mbf{x}),
\eeq
where
\beq
\mbf{e}(\mbf{x}) = \bma{c}{\mbf{e}_0(\mbf{x})} \\ {\mbf{e}_{u,1}(\mbf{x})} \\ \vdots \\ {\mbf{e}_{u,K}(\mbf{x})} \\ {\mbf{e}_{y,0}(\mbf{x})} \\ \vdots \\ {\mbf{e}_{y,K}(\mbf{x})}  \ema, \begin{array}{c} {\mbf{e}_0(\mbf{x}) = \mbf{x}_0 - \mbfcheck{x}_0} ,\\ \\ {\mbf{e}_{u,k}(\mbf{x})  = \mbf{x}_k - \mbf{f}(\mbf{x}_{k-1}, \mbf{u}_{k-1}),}  \\ \\ {\mbf{e}_{y,k}(\mbf{x}) = \mbf{y}_k - \mbf{g}(\mbf{x}_k),} \\\\ \mbf{W} = \mathrm{diag}({\mbfcheck{P}_0^{-1}},{\mbf{Q}_0^{-1},\ldots,\mbf{Q}_{K-1}^{-1}}, \\ \hspace{3cm} {\mbf{R}_0^{-1},\ldots,\mbf{R}_K^{-1}}). \end{array}\label{eq:map6} 
\eeq
This is known as the full \emph{batch nonlinear least-squares} estimator. The optimization problem \eqref{eq:map2} can be solved iteratively using the \emph{Gauss-Newton} algorithm, which starts with an initial state estimate $\mbf{x}^{(0)}$, and computes an update $\delta \mbf{x}^{(\ell)}$ to the $\ell^\mathrm{th}$ iteration by solving
\beq \label{eq:gn1}
(\mbf{H}^\trans \mbf{W} \mbf{H}) \delta \mbf{x}^{(\ell)} = -\mbf{H}^\trans \mbf{W} \mbf{e}(\mbf{x}^{(\ell)}), 
\eeq
where 
\bdis 
\mbf{H} = \left. \frac{\p \mbf{e}(\mbf{x})}{\p \mbf{x}} \right|_{\mbf{x}^{(\ell)}} = \bma{cccc} \mbf{1} &&& \\ -\mbf{A}_0 & \mbf{1} && \\ &\ddots &\ddots & \\ && -\mbf{A}_{K-1} & \mbf{1}\\ -\mbf{C}_0 &&&  \\& -\mbf{C}_1 &&  \\ && \ddots & \\  &&& -\mbf{C}_K \ema,
\edis
where $\mbf{A}_k$ and $\mbf{C}_k$ are the Jacobians of the process and measurement models, respectively, with respect to the state $\mbf{x}_k$. The state estimate is updated with $\mbf{x}^{(\ell+1)} = \mbf{x}^{(\ell)} + \delta \mbf{x}^{(\ell)}$, and the process is repeated until convergence. In practice, variants such as the \emph{Levenberg-Marquart} algorithm, are used to improve convergence.

In this batch framework, local observability of the system defined by \eqref{eq:map3}-\eqref{eq:map4} amounts to showing that the matrix $\mbf{H}^\trans \mbf{W} \mbf{H}$ is full rank when the prior is ignored. In \cite[Ch.~3.1]{Barfoot2019}, it is shown that this culminates in the equivalent requirement that 
\beq
\mathrm{rank} \left[ \mbf{C}_0^\trans \;\;\; \mbf{A}_0^\trans \mbf{C}_1^\trans \;\;\; \mbf{A}_0^\trans \mbf{A}_1^\trans \mbf{C}_2^\trans \;\;\; \ldots \;\;\; \mbf{A}_0^\trans \ldots \mbf{A}_{K-1}^\trans \mbf{C}_{K}^\trans \right] = n. \label{eq:map5}
\eeq

\subsection{Sliding Window Filtering}
The \emph{sliding window filter} is a batch estimation framework with constant time and memory requirements, achieving so by \emph{marginalizing out} older states \cite{Sibley2006,Dong-Si2011a}. Consider a scenario where a robot travels until time $t_{k_1}$, at which point it performs a full batch estimate of its state history, to produce $\mbfhat{x}_{0:k_1}$. It then continues to travel until time $t_{k_2}$, adding the new states $\mbf{x}_{k_1 + 1:k_2}$ to its state history. The $m$ oldest states $\mbf{x}_{0:m-1}$ are then marginalized out, thus removing them from the optimization problem being solved at $t_{k_2}$. The remaining states from the previous window's estimate are $\mbf{x}_{m:k_1}$.
\begin{align*}
    \overbrace{\hspace{4cm}}^{\text{new window of length }K= k_2 - m + 1} \\
\underbrace{\mbf{x}_0\;\; \mbf{x}_1 \; \ldots \; \mbf{x}_{m -1} \;\; \mbf{x}_{m} \; \ldots \; \mbf{x}_{k_1}}_{\text{old window of length }K = k_1 + 1}\;\; \mbf{x}_{k_1+1} \;\ldots \; \mbf{x}_{k_2}
\end{align*}
The joint PDF of the new window, given the entire history of measurements until $t_{k_2}$, can be written as
{
\begin{multline}
  p( \mbf{x}_{m:k_2}| \mbfcheck{x}_0, \mbf{u}_{0:k_2-1}, \mbf{y}_{0:k_2})  =\\ \eta p(\mbf{x}_{m+1:k_2}|\mbf{u}_{m:k_2-1}, \mbf{y}_{m:k_2}, \mbf{x}_m) p(\mbf{x}_{m} | \mbfcheck{x}_0, \mbf{u}_{0:m-1}, \mbf{y}_{0:m -1}).  \label{eq:swe3} 
\end{multline}}
The crux of the marginalization process is to determine an expression for $p(\mbf{x}_{m} | \mbfcheck{x}_0, \mbf{u}_{0:m-1}, \mbf{y}_{0:m -1})$, which takes the role of the new ``prior'' distribution of the new window. Performing marginalization properly, as opposed to naively removing the oldest states from the current estimation problem, is critical to maintaining an accurate, consistent estimator. To do this, consider instead the following PDF, for which an analytical expression is available,
\begin{multline} \label{eq:swe2}
p(\mbf{x}_{0:m}| \mbfcheck{x}_0, \mbf{u}_{0:m-1}, \mbf{y}_{0:m -1}) \\= \beta \exp(-\frac{1}{2}\mbf{e}_m(\mbf{x}_{0:m})^\trans \mbf{W}_m \mbf{e}_m(\mbf{x}_{0:m})),
\end{multline}
where 
\begin{align*}
\mbf{e}_m(\mbf{x}_{0:m}) &= \bma{ccccccc} \mbf{e}_0^\trans & \mbf{e}_{u,1}^\trans & \ldots& \mbf{e}_{u,m}^\trans& \mbf{e}_{y,0}^\trans& \ldots& \mbf{e}_{y,m-1}^\trans\ema^\trans,\\
\mbf{W}_m &= \mathrm{diag}(\mbfcheck{P}_0^{-1}, \mbf{Q}_0^{-1}, \ldots, \mbf{Q}_{m-1}^{-1}, \mbf{R}_0^{-1}, \ldots, \mbf{R}_{m-1}^{-1}),
\end{align*}
$\beta$ is a normalization constant, and the $\mbf{e}_0, \mbf{e}_{u,i}, \mbf{e}_{y,i}$ terms are defined in \eqref{eq:map6}, but written without arguments. This is, in general, not a Gaussian distribution due to the nonlinear nature of $\mbf{e}_m(\cdot)$, but it can be approximated as one by linearizing $\mbf{e}_m(\cdot)$ about a subset of the state estimates obtained from the previous window, being $\mbfhat{x}_{0:m} $. Substituting a first-order approximation of $\mbf{e}_m(\cdot)$ into \eqref{eq:swe2}, and after some manipulation, the mean and covariance of a Gaussian approximation to $p(\mbf{x}_{0:m}| \mbfcheck{x}_0, \mbf{u}_{0:m-1}, \mbf{y}_{0:m -1}) \approx \mc{N}(\mbs{\mu}_{0:m}, \mbs{\Sigma}_{0:m})$ can be shown to be
\begin{align}
    \mbs{\mu}_{0:m} &= \bma{c} \mbs{\mu}_{0:m-1} \\ {\mbs{\mu}_{m}} \ema \nonumber \\
    &= \mbfhat{x}_{0:m}  - (\mbf{H}^\trans_m \mbf{W}_m \mbf{H}_m)^{-1}\mbf{H}_m^\trans \mbf{W}_m \mbf{e}_m(\mbfhat{x}_{0:m}) ,\label{eq:swf5} \\
    \mbs{\Sigma}_{0:m} &=\bma{cc} \mbs{\Sigma}_{0:m-1} & \mbs{\Sigma}_{0:m-1,m}  \\ \mbs{\Sigma}_{m,0:m-1} & {\mbs{\Sigma}_{m}} \ema= (\mbf{H}^\trans_m \mbf{W}_m \mbf{H}_m)^{-1}, \label{eq:swf3}
\end{align}
respectively, where 
\beq
\mbf{H}_m = \left.\frac{\p \mbf{e}_m(\mbf{x})}{\p \mbf{x}} \right|_{\mbfhat{x}_{0:m}}.  \label{eq:swf1}
\eeq
It finally follows that 
\bdis
p(\mbf{x}_{m} | \mbfcheck{x}_0, \mbf{u}_{0:m-1}, \mbf{y}_{0:m -1}) \approx \mc{N}(\mbs{\mu}_{m}, \mbs{\Sigma}_{m}).
\edis
The specific step in which the marginalization occurred is when $\mbs{\mu}_m$ and $\mbs{\Sigma}_m$ are extracted from \eqref{eq:swf5} and \eqref{eq:swf3}, respectively.  With an expression for the new prior now identified, one may return to \eqref{eq:swe3} to construct a nonlinear least-squares problem in the exact same manner as the full batch problem, and solve it with the Gauss-Newton or Levenberg-Marquart algorithm.

\section{Proposed Algorithm}\label{sec:algorithm}
\begin{figure}[t]
    \includegraphics[width = \linewidth]{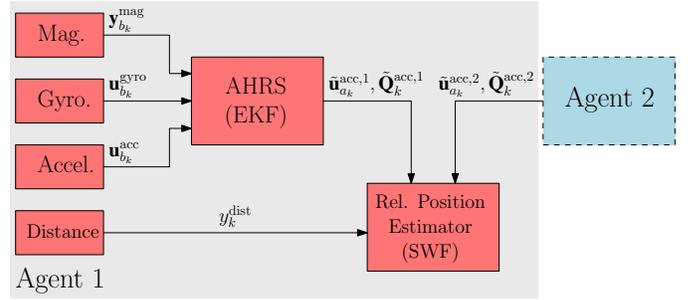}
    \caption{High-level diagram of the architecture of the proposed algorithm. Each agent contains an AHRS that uses on-board accelerometer, rate gyro, and magnetometer measurements.}
    \label{fig:block_architecture}
\end{figure}
\subsection{High-level architecture}
The proposed algorithm consists of two main components, with the architecture displayed graphically in Figure \ref{fig:block_architecture}. Although each agent can execute the following algorithm, it will be explained from the perspective of Agent 1.
\begin{enumerate}
    \item {\textbf{Attitude \& Heading Reference System (AHRS)} \\ At each IMU measurement, the agents calculate their own attitudes relative to a common local frame, such as an East-North-Up reference frame, denoted $\rframe{a}$. Using the attitude estimate, Agent 2 communicates an estimate of their translational acceleration, resolved in $\rframe{a}$, along with a corresponding covariance.}
    \item {\textbf{Relative Position Estimator (RPE)}\\
    Using the translational acceleration information of Agent 2 communicated to Agent 1, as well as a set of previous distance measurements, a sliding window filter determines the position and velocity of Agent 1 relative to Agent 2 in $\rframe{a}$. The high-frequency acceleration information of both agents is integrated between a set of optimized \emph{keypoints}, which will be explained in Section \ref{sec:rpe}.}
\end{enumerate}
The primary contribution of this letter is the development and testing of the RPE. The choice of using a sliding window filter instead of a one-step-ahead filter, such as an extended Kalman filter (EKF), stems from the fact that position and velocity states are instantaneously unobservable given one distance measurement. However, as will be shown in Section \ref{sec:obsv}, the system is observable when multiple distance measurements are used for state estimation, which is what the sliding window filter does.

The cascaded architecture involving the AHRS and the RPE is \emph{loosely-coupled} due to the separation of attitude and translational state estimators. This choice comes with several advantages and disadvantages, when compared to a \emph{tightly-coupled} framework that would estimate the attitude and translational states in one estimator. The main disadvantage is that any coupling information between the AHRS and RPE states is lost. {This point represents the largest possible source of inaccuracy, as it results in the RPE performance being highly dependent on the accuracy of the attitude estimates.} However, this cost is justified by the following advantages.
\begin{itemize}
    \item The corrective step of the AHRS can be executed at an arbitrarily high frequency, thus enabling the incorporation of all possible magnetometer and accelerometer measurements for attitude correction.
    \item The processing can easily be distributed among the two agents, {and communication requirements are heavily reduced as gyroscope and magnetometer measurements do not need to be shared.}
    \item The relative position dynamics reduce to a linear model, which is convenient from both mathematical and computational points of view. 
\end{itemize}
{\subsection{Attitude \& Heading Reference System}}
The AHRS used in the presented simulations and experiments is an extended Kalman filter (EKF) very similar to \cite[Ch.~10]{Farrell2008}, but with some modifications based on \cite{Barrau2017}. The AHRS estimates the agent attitude at time step $k$, denoted as $\mbfhat{C}_{ab_k}$. The estimate also has a corresponding covariance $\mbf{P}^{\mathrm{ahrs}}_k$, defined such that $\mbf{C}_{ab_k} = \mbfhat{C}_{ab_k}\exp(\delta \mbs{\phi}^\times)$, where $\delta \mbs{\phi} \sim \mc{N}(\mbf{0},\mbf{P}^{\mathrm{ahrs}}_k)$ and $\mbf{C}_{ab_k}$ is the true agent attitude.
The translational acceleration vector $\mbftilde{u}^{\mathrm{acc}}_{a_k}$, resolved in $\rframe{a}$, can be calculated using the vehicle attitude and the raw accelerometer measurements $\mbf{u}^{\mathrm{acc}}_{b_k}$ using
\beq \label{eq:ahrs1}
\mbftilde{u}^{\mathrm{acc}}_{a_k} = \mbf{C}_{ab_k} \mbf{u}^{\mathrm{acc}}_{b_k} + \mbf{g}_a,
\eeq
where $\mbf{g}_a$ is the gravity vector resolved in $\rframe{a}$. However, a covariance associated with $\mbftilde{u}^{\mathrm{acc}}_{a_k}$ is also required. Equation \eqref{eq:ahrs1} is a nonlinear function of two random variables, $\mbf{C}_{ab_k}$ and $\mbf{u}^{\mathrm{acc}}_{b_k}$, and as such the translational acceleration distribution can be approximated as 
$\mbftilde{u}^{\mathrm{acc}}_{a_k} \sim \mc{N}(\mbfhat{C}_{ab_k} \mbf{u}^{\mathrm{acc}}_{b_k} + \mbf{g}_a, \mbftilde{Q}^{\mathrm{acc}}_k)$ where
 \bdis
 \mbftilde{Q}^{\mathrm{acc}}_k = \mbf{M} \mbf{Q}^{\mathrm{acc}}_k\mbf{M}^\trans + \mbf{G} \mbf{P}^{\mathrm{ahrs}}_k \mbf{G}^\trans, 
 \edis
 and $\mbf{Q}^{\mathrm{acc}}_k$ is the covariance associated with the raw accelerometer measurements. The matrices $\mbf{M} = \mbfhat{C}_{ab_k}$ and $\mbf{G} = -  \mbfhat{C}_{ab_k}\mbf{u}^{\mathrm{acc}^\times}_{b_k}$ are the Jacobians of \eqref{eq:ahrs1} with respect to $ \mbf{u}^{\mathrm{acc}}_{b_k}$ and $\delta \mbs{\phi}$, respectively. Although not specified in the notation used in this section, each of the two agents computes their own, independent estimates of $\mbftilde{u}^{\mathrm{acc}}_{a_k} $ and $\mbftilde{Q}^{\mathrm{acc}}_k $. They are both communicated to the relative position estimator, consisting of 3 numbers for the $\mbftilde{u}^{\mathrm{acc}}_{a_k} $ and 6 numbers for the symmetric covariance matrix $\mbftilde{Q}^{\mathrm{acc}}_k $.

\subsection{Relative Position Estimator}\label{sec:rpe}

Let $\mbf{r}^{12}_{a_k}$ be the position of Agent 1, relative to Agent 2, resolved in $\rframe{a}$. Let $\mbf{v}^{12}_{a_k}$ be the velocity of Agent 1, relative to Agent 2, with respect to $\rframe{a}$. The relative position estimator, which will be on-board Agent 1, estimates the state $\mbf{x}_k =  [\mbf{r}^{12^\trans}_{a_k} \; \mbf{v}^{12^\trans}_{a_k}]^\trans$. The translational acceleration and corresponding covariances of Agents 1 and 2, being $(\mbftilde{u}^{\mathrm{acc},1}_{a_k},\mbftilde{Q}^{\mathrm{acc},1}_k)$ and $(\mbftilde{u}^{\mathrm{acc},2}_{a_k},\mbftilde{Q}^{\mathrm{acc},2}_k)$, are combined to produce a relative acceleration estimate, $\mbf{u}_k \triangleq \mbftilde{u}^{\mathrm{acc},1}_{a_k} - \mbftilde{u}^{\mathrm{acc},2}_{a_k},$ with covariance $ \mbftilde{Q}_k = \mbftilde{Q}^{\mathrm{acc},1}_k + \mbftilde{Q}^{\mathrm{acc},2}_k$.
This allows the relative position and velocity process model to be written as
\beq
\mbf{x}_k = \mbf{A}_{k-1} \mbf{x}_{k-1} + \mbf{B}_{k-1} \mbf{u}_{k-1} + \mbf{w}_{k-1},\label{eq:rpe1}
\eeq 
where $\Delta t_k = t_k - t_{k-1}$, $\mbf{w}_{k-1} \sim \mc{N}(\mbf{0}, \mbf{Q}_{k-1})$, 
\bdis
\mbf{A}_{k-1} = \bma{cc} \mbf{1} & \Delta t_{k} \mbf{1} \\ \mbf{0} & \mbf{1} \ema , \quad \mbf{B}_{k-1}= \bma{c} (\Delta t_{k}^2/2) \mbf{1} \\ \Delta t_{k} \mbf{1} \ema,
\edis
\bdis
\mbf{Q}_{k-1} = \bma{cc} \frac{1}{3}\Delta t_k^3 \mbftilde{Q}_{k-1} & \frac{1}{2} \Delta t_k^2 \mbftilde{Q}_{k-1} \\\frac{1}{2} \Delta t_k^2 \mbftilde{Q}_{k-1} & \Delta t_k \mbftilde{Q}_{k-1} \ema.
\edis
 Furthermore, the RPE uses distance measurements, which have a measurement model of the form
\beq
y^{\mathrm{dist}}_k = \norm{\mbf{r}^{12}_{a_k}} + v_k, \label{eq:rpe3}
\eeq
where $v_k \sim \mc{N}(0, R_k)$. The Jacobian of \eqref{eq:rpe3} with respect to $\mbf{x}_k$ is given by $\mbf{C}_k = \bma{cc} \mbs{\rho}_{k}^\trans & \mbf{0} \ema,$  where  $ \mbs{\rho}_k = \mbf{r}^{12}_{a_k} / \norm{\mbf{r}^{12}_{a_k}}$.
\begin{figure*}[t]
    \centering
    \includegraphics[trim = {16cm 6cm 16cm 5.9cm}, clip = true, width = 0.33333\textwidth]{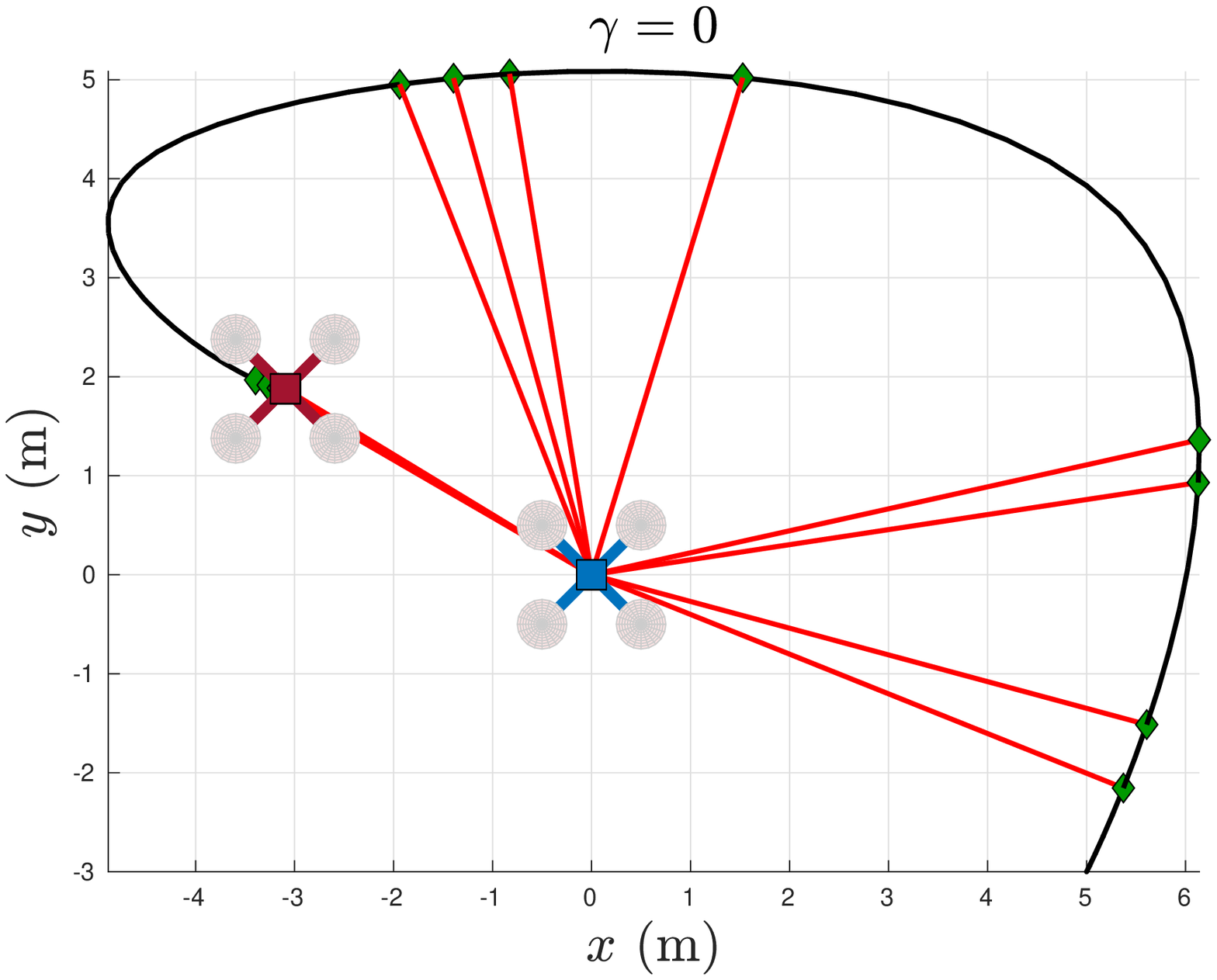}%
    \includegraphics[trim = {16cm 6cm 16cm 5.9cm}, clip = true, width = 0.33333\textwidth]{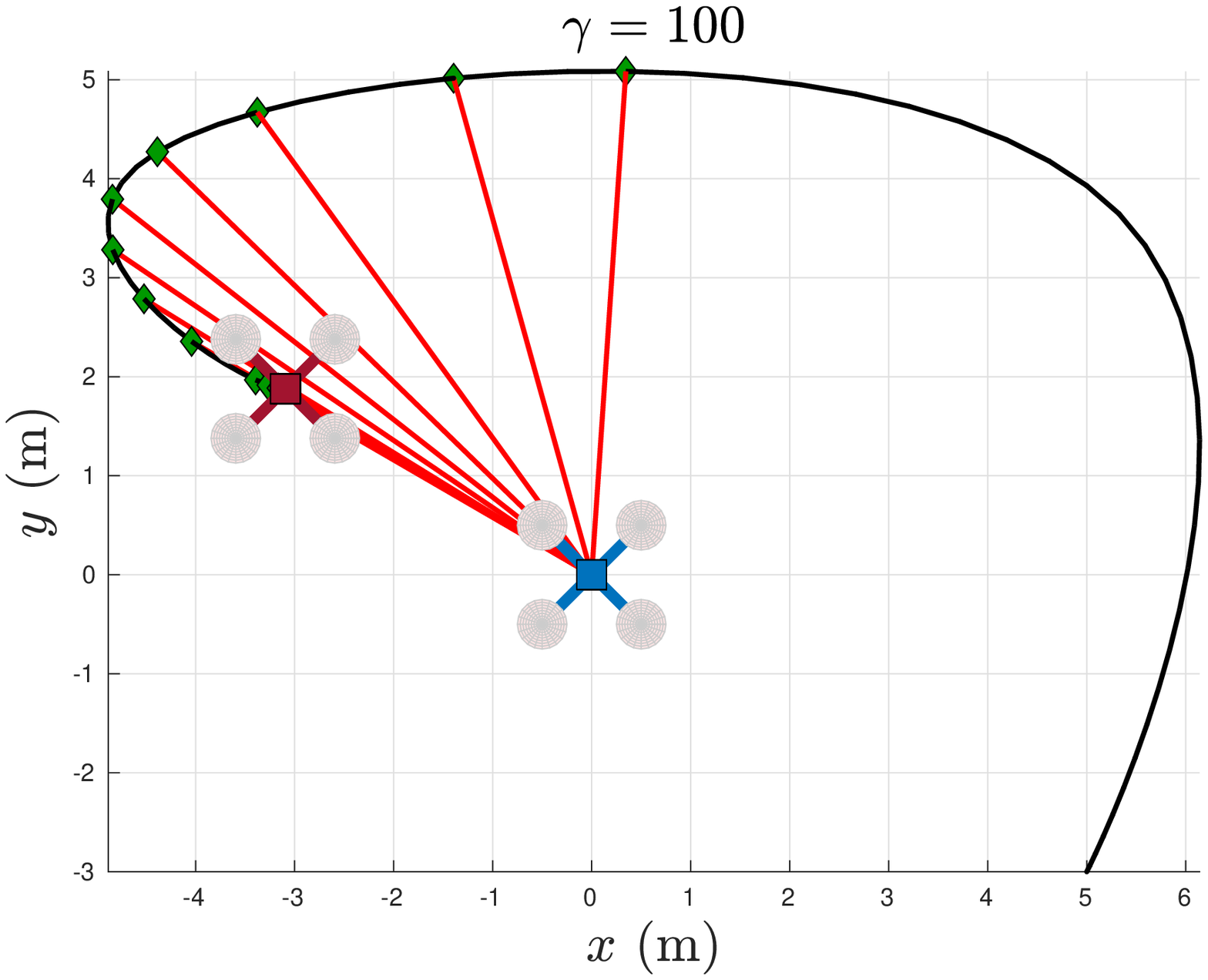}%
    \includegraphics[trim = {16cm 6cm 16cm 5.9cm}, clip = true, width = 0.33333\textwidth]{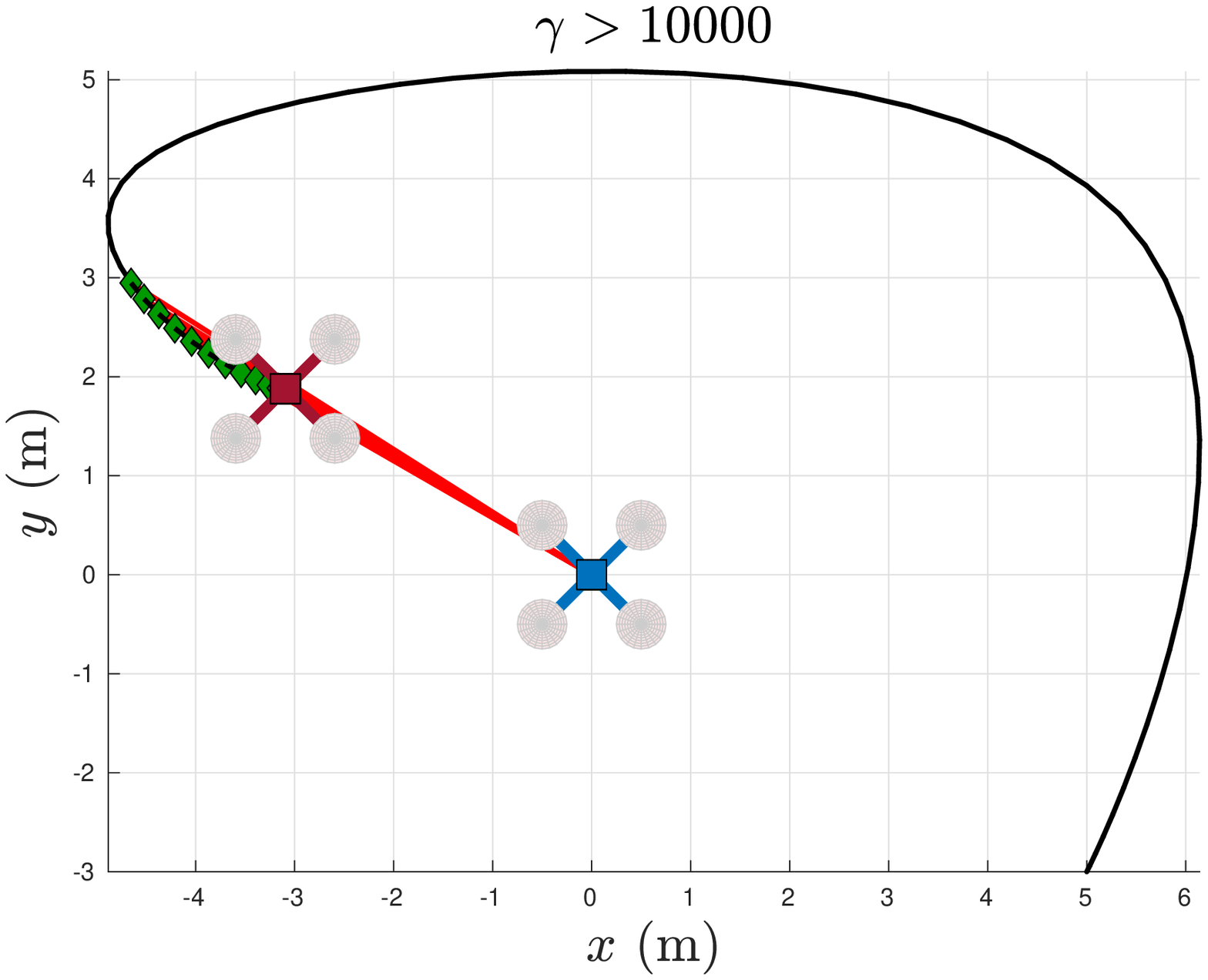}
    \caption{Effect of the parameter $\gamma$ on the keypoint selection with $K = 12$, and the old trajectory shown as the black line. (left) With $\gamma = 0$, the algorithm greedily minimizes the GDOP, and results in keypoints placed far apart in time. (middle) With a moderate value for $\gamma$, the algorithm obtains a balance between pure GDOP minimization, and making all the keypoints as close as possible. (right) With a large value for $\gamma$, the minimal cost is the solution with  the keypoints as close as possible, and the ``plain vanilla'' sliding window filter is recovered.}
    \label{fig:keypoints}
 \end{figure*} 

A naive implementation of the sliding window filter would result in including a state at every single measurement. With modern IMUs, this typically occurs at a frequency of 100~Hz - 1000~Hz, which can result in an enormous amount of states in the window if it were to span only a few seconds. The number of states in the window can be reduced by omitting all except a select few called \emph{keypoints} \cite{Leutenegger2015a}. The keypoints are identified as a set of $K$ distinct time indices $\mc{K} = \{p_0, p_1, \ldots, p_{K-1}\}$ such that the states $\mbf{x}_{p_0}, \mbf{x}_{p_1}, \ldots, \mbf{x}_{p_{K-1}}$ are the only states estimated inside the window. To do this, a process model of some sort is required that relates successive keypoint states.

A technique called \emph{pre-integration} expresses the process model between two non-adjacent states $\mbf{x}_i$ and $\mbf{x}_j$ by a single relation. The use of the linear model in \eqref{eq:rpe1} makes pre-integration trivial, as direct iteration of \eqref{eq:rpe1} leads to
\begin{align}
\mbf{x}_j &= \left(\prod_{k = i}^{j-1} \mbf{A}_{k}\right) \mbf{x}_i + \sum_{k = i}^{j-1} \left(\prod_{\ell = k+1}^{j-1} \mbf{A}_\ell\right) \mbf{B}_k \mbf{u}_{k},\\ \label{eq:rpe2}
&\triangleq \mbf{A}_{ji} \mbf{x}_i + \mbf{b}_{ji}. 
\end{align}
Importantly, the terms $\mbf{A}_{ji}$ and $\mbf{b}_{ji}$ do not need to be recomputed at each Gauss-Newton iteration of the sliding window filter, since they are independent of the state estimate. This yields substantial computational savings compared to a case where the process model is too complicated to easily pre-integrate. With this pre-integration ability, states in the sliding window filter can be placed arbitrarily far apart in time with a very minor increase in computation time. 

As such, this letter constructs a sliding window filter where each state is a keypoint, successive keypoint states are related by the pre-integrated process model \eqref{eq:rpe2}, and the measurement model at each keypoint is given by \eqref{eq:rpe3}. The ability to place keypoints at arbitrary locations in an agent's state history now gives rise to an entire design problem, which is to design an appropriate keypoint selection strategy. This letter will propose one of many possible strategies.

\subsection{Observability Analysis} \label{sec:obsv}
Evaluated at arbitrary keypoints, the observability rank condition given in \eqref{eq:map5} requires that 
\beq
\mathrm{rank}(\mc{O}) = 6, \label{eq:obs1}
\eeq 
where
\begin{multline*} 
\hspace{-0.3cm} \mc{O} =\\ \bma{ccccc} \mbs{\rho}_{p_0} & \mbs{\rho}_{p_1} & \mbs{\rho}_{p_2} & \cdots &\mbs{\rho}_{p_{K-1}} \\
    \mbf{0} & \Delta t_{p_1 p_0} \mbs{\rho}_{p_1} &\Delta t_{p_2 p_0} \mbs{\rho}_{p_2}&\cdots &\Delta t_{p_{K-1} p_0} \mbs{\rho}_{p_{K-1}}\ema,
\end{multline*}
$\mbs{\rho}_k = \mbf{r}^{12}_{a_k} / \norm{\mbf{r}^{12}_{a_k}}$, and $\Delta t_{ji} = t_j - t_i$. To ensure observability, the keypoints must be chosen to satisfy \eqref{eq:obs1}. 

\begin{theorem} \label{thm:observability} Sufficient conditions to satisfy \eqref{eq:obs1} are
\begin{enumerate}
    \item $K \geq 6$,
    \item $\Delta t_{p_1 p_0} \neq \Delta t_{p_2 p_0} \neq \ldots \Delta t_{p_K p_0} \neq 0$,
    \item and that there exists two non-intersecting subsets of $\{\mbs{\rho}_{p_0},\;\mbs{\rho}_{p_1},\;\ldots, \;\mbs{\rho}_{p_{K-1}}\}$, each composed of 3 linearly independent elements. 
\end{enumerate}
\end{theorem}
\begin{proof}
    See Appendix \ref{sec:observability_proof}.
\end{proof}
These conditions result in the requirement that the agents' relative motion be non-planar, and that  keypoints are at distinct times.
\subsection{Keypoint Selection Strategy}
\begin{algorithm}[t]
	\caption{Greedy Keypoint Selection. Given the current time step $k$, the total number of keypoints $K$, and the keypoint index set of the previous window $\mc{K}_{\mathrm{old}}$, calculate the new keypoint index set $\mc{K}$.  A set of candidate indices $\mc{C}$ is maintained.} 
	\label{alg:keypoints}
	\begin{algorithmic}[1]
        \Function {GetKeypoints}{$k, K ,\mc{K}_{\mathrm{old}}$}
            \State $\mc{K} \gets \{k-3, k-2, k-1, k\}$
			\State $\mc{C} \gets \mc{K}_{\mathrm{old}} \cup \{ \max(\mc{K}_{\mathrm{old}}) + 1, \ldots, k - 4\}$
            \For{$i = 1,\ldots, K-4$}
                \State $J_{\mathrm{best}} \gets \infty$
                \ForAll {$p \in \mc{C}$}
                    \State $\mc{K}_{\mathrm{temp}} \gets \mc{K} \cup \{p\}$
                    \If{$J(\mc{K}_{\mathrm{temp}}) \leq J_{\mathrm{best}}$}
                        \State $J_{\mathrm{best}} \gets J(\mc{K}_{\mathrm{temp}})$
                        \State $p_{\mathrm{best}} \gets p$
                    \EndIf
                \EndFor
                \State $\mc{K} \gets \mc{K} \cup \{ p_{\mathrm{best}}\}$
                \State $\mc{C} \gets \mc{C} \setminus \{ p_{\mathrm{best}}\}$
            \EndFor
			\State \Return $\mc{K}$
		\EndFunction 
    \end{algorithmic} 
    \vspace{-0.05cm}
\end{algorithm}
The keypoint selection strategy aims to find a suitable keypoint index set $\mc{K}$. The proposed strategy is a greedy algorithm, which places keypoints one-by-one, searching linearly through all possible candidate keypoint locations, and selecting the location with the lowest cost. The chosen cost function is based on the \emph{geometric dilution of precision} (GDOP) \cite[Ch.~8.5]{Farrell2008}, and is given by
\beq
J(\mc{K}) = \trace\left((\mbf{D}(\mc{K})^\trans \mbf{D}(\mc{K}))^{-1}\right) + \gamma \sum_{i = 1}^{|\mc{K}|} \Delta t_{p_i p_{i-1}}, \label{eq:kp1}
\eeq
where $\mbf{D}(\mc{K}) = \bma{ccccc} \mbs{\rho}_{p_0} & \mbs{\rho}_{p_1} & \mbs{\rho}_{p_2} & \cdots &\mbs{\rho}_{p_{|\mc{K}|}} \ema^\trans$, $|\mc{K}|$ denotes the number of elements in $\mc{K}$, and $\gamma$ is a tuning parameter. The first term in \eqref{eq:kp1} is the square of the GDOP, also noting that the form of the $\mbf{D}(\mc{K}) $ matrix shows a natural similarity to the observability matrix shown in \eqref{eq:obs1}. Minimizing this first term avoids a keypoint selection that will make the observability matrix close to rank deficient. The second term in \eqref{eq:kp1} penalizes the keypoints for being placed too far apart in time. Excessively old keypoints will require long pre-integration between their adjacent states, which increases the covariance of the estimate. The tuning parameter $\gamma$ is used to balance the two terms, and its effect is illustrated in Figure \ref{fig:keypoints}.
\begin{figure}[t]
    \includegraphics[width = \linewidth]{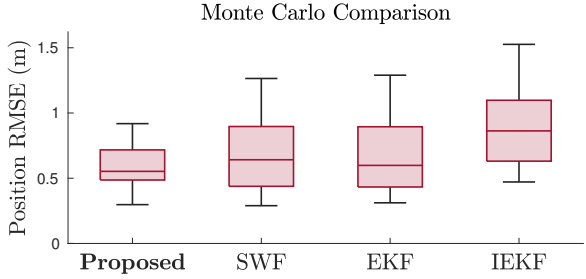}
    \caption{Results of 200 Monte Carlo trials. The proposed method offers an enhancement over the ``plain vanilla'' implementation of the sliding window filter (SWF), which simply places the keypoints at the most recent range measurements, as well as a standard extended Kalman filter (EKF), and an Iterated EKF (IEKF).}
    \label{fig:monte_carlo_box}
\end{figure}
Finally, it is also possible to recursively compute the matrix $\mbs{\Lambda} = (\mbf{D}(\mc{K}) ^\trans \mbf{D}(\mc{K}) )^{-1}$, which avoids recomputing the inverse when a new keypoint is added to $\mc{K}$. Defining $\tilde{\mc{K}} = \mc{K} \cup \{p\}$, it can be shown that
\begin{align}
\left(\mbf{D(\tilde{\mc{K}})}^\trans \mbf{D(\tilde{\mc{K}})}\right)^{-1} &= \left(\mbf{D({\mc{K}})}^\trans \mbf{D({\mc{K}})} + \mbs{\rho}_p \mbs{\rho}_p^\trans\right)^{-1}\\
&= \mbs{\Lambda} - \frac{\mbs{\Lambda}\mbs{\rho}_p\mbs{\rho}_p^\trans\mbs{\Lambda}}{1 +  \mbs{\rho}_p^\trans\mbs{\Lambda}\mbs{\rho}_p}, \label{eq:kp2}
\end{align}
where, in \eqref{eq:kp2}, the Woodbury matrix identity is used. Equation \eqref{eq:kp2} is particularly useful when evaluating the \textbf{if} statement in line 8 of Algorithm \ref{alg:keypoints}, which shows the details of the proposed keypoint placement strategy. The algorithm initializes $\mc{K}$ with the 4 most recent indices in order to produce an  invertible $(\mbf{D}(\mc{K})^\trans \mbf{D}(\mc{K}))^{-1}$ matrix, after which the cost function naturally leads to a selection of older keypoints.

\begin{figure}[t]
    \includegraphics[width = \linewidth]{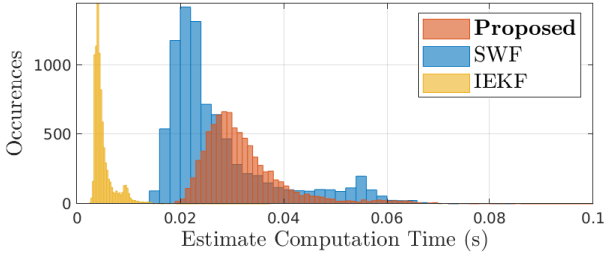}
    \caption{Histogram of the computation time of the different algorithms, when simulated on a laptop with an Intel Core i7-9750H CPU.}
    \label{fig:comp_time}
\end{figure}
\begin{figure}[t]
    \includegraphics[width = \linewidth,clip = true, trim = {14.5cm 4cm 14cm 3.5cm}]{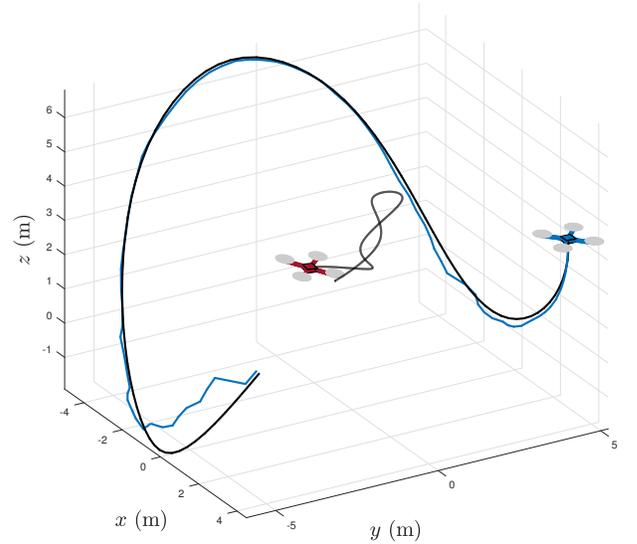}
    \caption{Simulation of the proposed estimation algorithm. The ground truth trajectory is shown in black, and the estimated trajectory in blue.}
    \label{fig:two_agent_sim_traj}
\end{figure}

\subsection{Multi-robot Scenario}
Although this paper focuses on a scenario with two agents, one way to extend the proposed algorithm to multiple agents is to naively copy the estimator for each pair of agents. In any case, it is possible to leverage the proposed framework to drastically simplify the inter-agent communication. By recalling that $\mbf{u}_k = \mbftilde{u}^{\mathrm{acc},1}_{a_k} - \mbftilde{u}^{\mathrm{acc},2}_{a_k}$, it is straightforward to show that $\mbf{b}_{ji}$ in \eqref{eq:rpe2} can be written as 
\bdis 
\mbf{b}_{ji} = \mbf{b}_{ji}^{\mathrm{acc},1} - \mbf{b}_{ji}^{\mathrm{acc},2},\edis
 where $\mbf{b}_{ji}^{\mathrm{acc},1}$ and $\mbf{b}_{ji}^{\mathrm{acc},2}$  only depend on measurements from Agents 1 and 2, respectively. This has a significant implication: the agents do not need to communicate their acceleration measurements at high frequency, but can instead preintegrate them in advance to produce $\mbf{b}_{ji}^{\mathrm{acc},1}$ and $\mbf{b}_{ji}^{\mathrm{acc},2}$, then communicate these quantities at an arbitrarily lower frequency. This is of substantial practical interest, as communicating all high-frequency IMU measurements could quickly exceed the inter-agent communication capacity, especially for multi-robot scenarios.

\section{Simulation Results}\label{sec:simulation}
The proposed algorithm is tested in simulation, and compared against a ``plain vanilla'' sliding window filter (\emph{Vanilla SWF} or \emph{SWF} in the figures), which simply uses the $K$ most recent range measurements as the keypoint times, as shown on the right of Figure \ref{fig:keypoints}. A comparison is also made to using a standard EKF as the relative position estimator.
The simulation is repeated for 200 Monte Carlo trials, where each trial uses a different, randomized trajectory. Table \ref{table:sim1} shows the values used when generating zero-mean Gaussian noise on each sensor.

Figure \ref{fig:monte_carlo_box} shows a box plot of the position root-mean-squared error (RMSE) for the 200 Monte Carlo trials, where the $\mathrm{RMSE}=\sqrt{(1/N)\sum_{k = 1}^{N} \mbf{e}_{r_k}^\trans \mbf{e}_{r_k}}$, $\mbf{e}_{r_k} = \mbfhat{r}^{12}_{a_k} - \mbf{r}^{12}_{a_k}$.

Using the greedy keypoint placement scheme shows a 9\% reduction in average position RMSE when compared to the Vanilla SWF,  7\%  compared to the EKF, and 32\% compared to the iterated EKF \cite[Ch.~4.2.5]{Barfoot2019}. An interesting observation is that EKF can perform as well, if not better than the sliding window filter and the iterated EKF, which both iterate the state estimate until a locally minimal least-squares cost is found. This implies that there exists a state with lower overall process and measurement error, yet higher true estimation error. Many instances have been observed where the least-squares cost decreases while true estimation error increases, a possible consequence of the unobservable nature of this problem. 


The proposed algorithm takes an average of 0.035~s per estimate computation, whereas the Vanilla SWF takes 0.027~s, the IEKF takes 0.005~s, and the EKF takes 0.0006~s. A histogram of the computation time is shown in Figure \ref{fig:comp_time}.

\section{Experimental Results} \label{sec:experiment}
The proposed algorithm is also tested in a real experiment, with two different hardware setups, both shown in Figure \ref{fig:baby_agent}. The first consists of a Raspberry Pi 4B, with an LSM9DS1 9-DOF IMU and a Pozyx UWB Developer Tag, providing distance measurements to an identical device. Although the cost of the actual sensors embedded in one of these prototypes is \$15 USD, each cost a total of approximately \$270 USD. However, this is mainly due to the cost of the Raspberry Pi 4B, the battery pack, and the Pozyx Developer Tag. The second setup consists of a Pozyx Developer tag mounted to quadrotors possessing a Pixhawk 4, which provides IMU and magnetometer data. In both setups, two identical devices are randomly moved around a room by hand, in an overall volume of roughly $5~\mathrm{m} \times 4~\mathrm{m} \times 2~\mathrm{m}$. When using the quadrotors, the motors are spinning without propellors to examine the effect of vibrations and magnetic interference. Ground truth position and attitude measurements are collected using an OptiTrack optical motion capture system. 
\begin{figure}[t]
    \includegraphics[width = \linewidth]{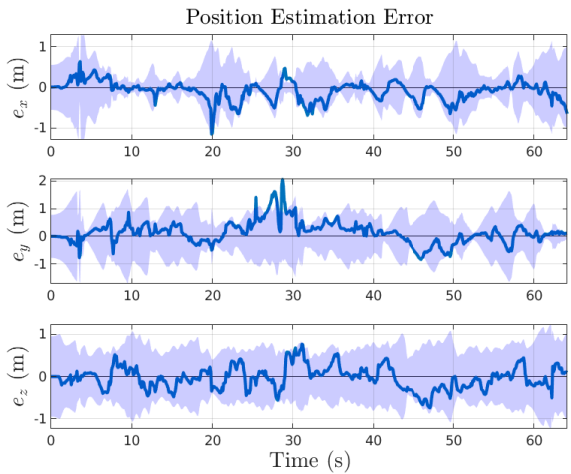}%
    \caption{Components of the position estimation error during an experimental trial, using the greedy keypoint placement method. The error compared to ground truth is shown as the blue line, with the shaded area representing the $\pm3 \sigma$ confidence bounds.}
    \label{fig:two_agent_greedy_pos_error}
\end{figure}
\begin{table}[t]
    \caption{Simulation Noise Properties}
    \vspace{-0.2cm}
    \begin{center}
    \begin{tabular}{c|c|c}
    \hline
    \textbf{Specification} & \textbf{Value} & \textbf{Units}\\
    \hline \hline
    Magnetometer std. dev. & 1 & $\mu \mathrm{F}$ \\
    Gyroscope std. dev. & 0.001 & $\mathrm{rad}/s$ \\
    Accelerometer std. dev.  & 0.01 & $\mathrm{m/s}^2$ \\
    Distance std. dev & 0.1 & $\mathrm{m}$\\
    Initial position std. dev. & $0.8$ & $\mathrm{m}$\\
    Initial velocity std. dev. & $0.1$ & $\mathrm{m/s}$\\
    Initial attitude std. dev. & $0.001$ & $\mathrm{red}$\\
    Accel/Gyro/Mag frequency & $100$ & $\mathrm{Hz}
    $\\
    Distance frequency & $10$ & $\mathrm{Hz}
    $\\
    RPE window size $K$ & $20$ & -\\
    RPE keypoint selection parameter $\gamma$ & $100$ & -\\
    RPE frequency & $10$ & $\mathrm{Hz}
    $\\
    \hline

    \end{tabular}
    \label{table:sim1}
    \end{center}
    \vspace{-8pt}
\end{table}

\begin{figure}[t]
    \centering
    \includegraphics[width = \linewidth]{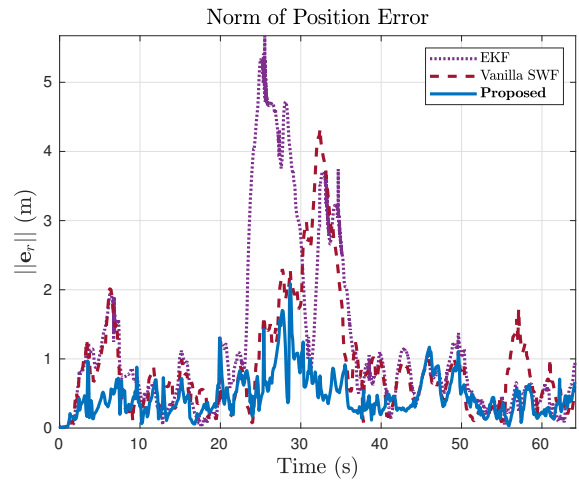}%
    \caption{Comparison of the norm of the position estimation error $\mbf{e}_r$, during an experimental trial. }
    \label{fig:two_agent_greedy_pos_norm}
\end{figure}
\begin{figure}[t]
    \centering
    \includegraphics[width = 0.49\linewidth]{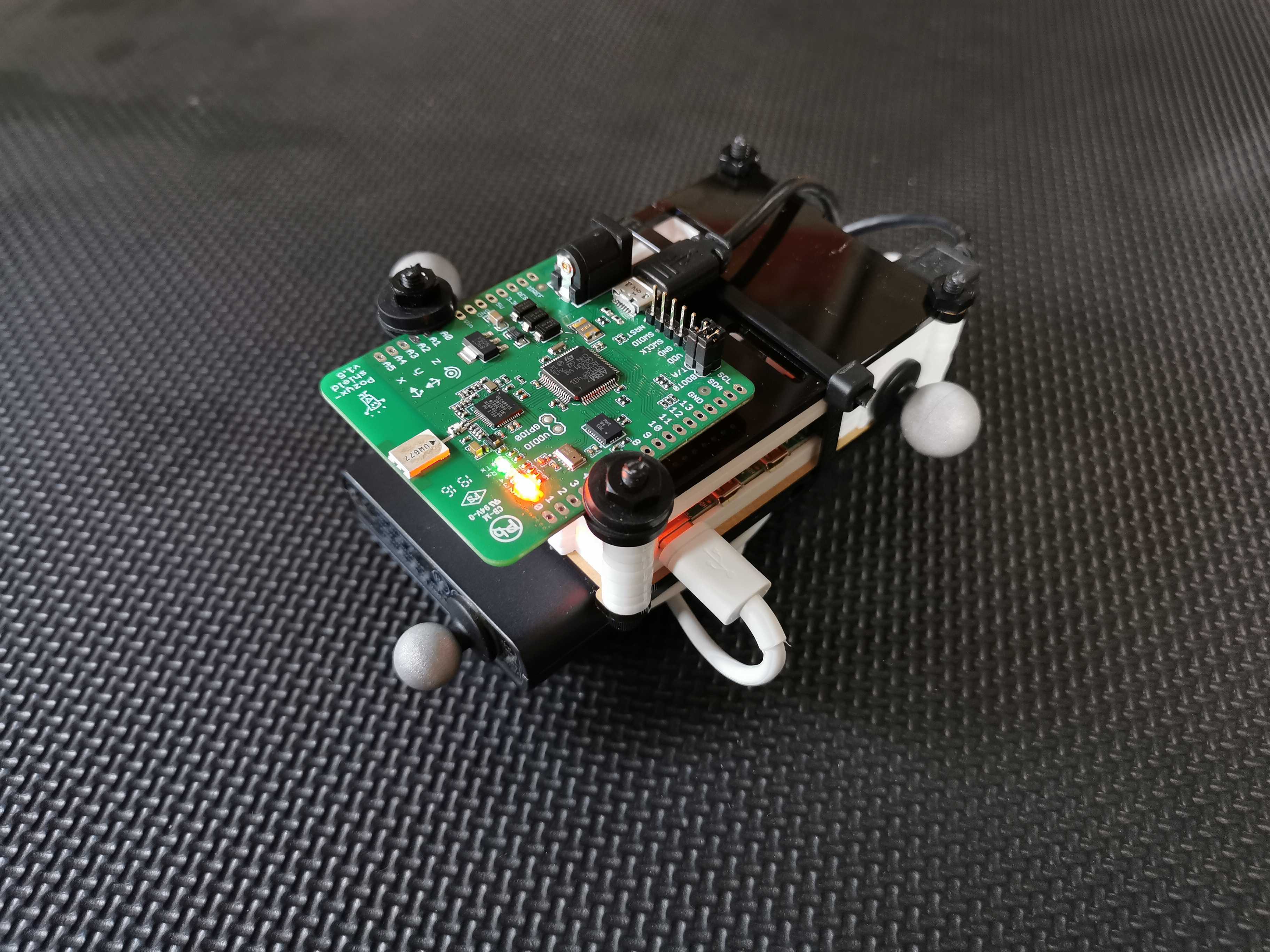}%
    \hspace{0.00001cm}
    \includegraphics[width = 0.49\linewidth]{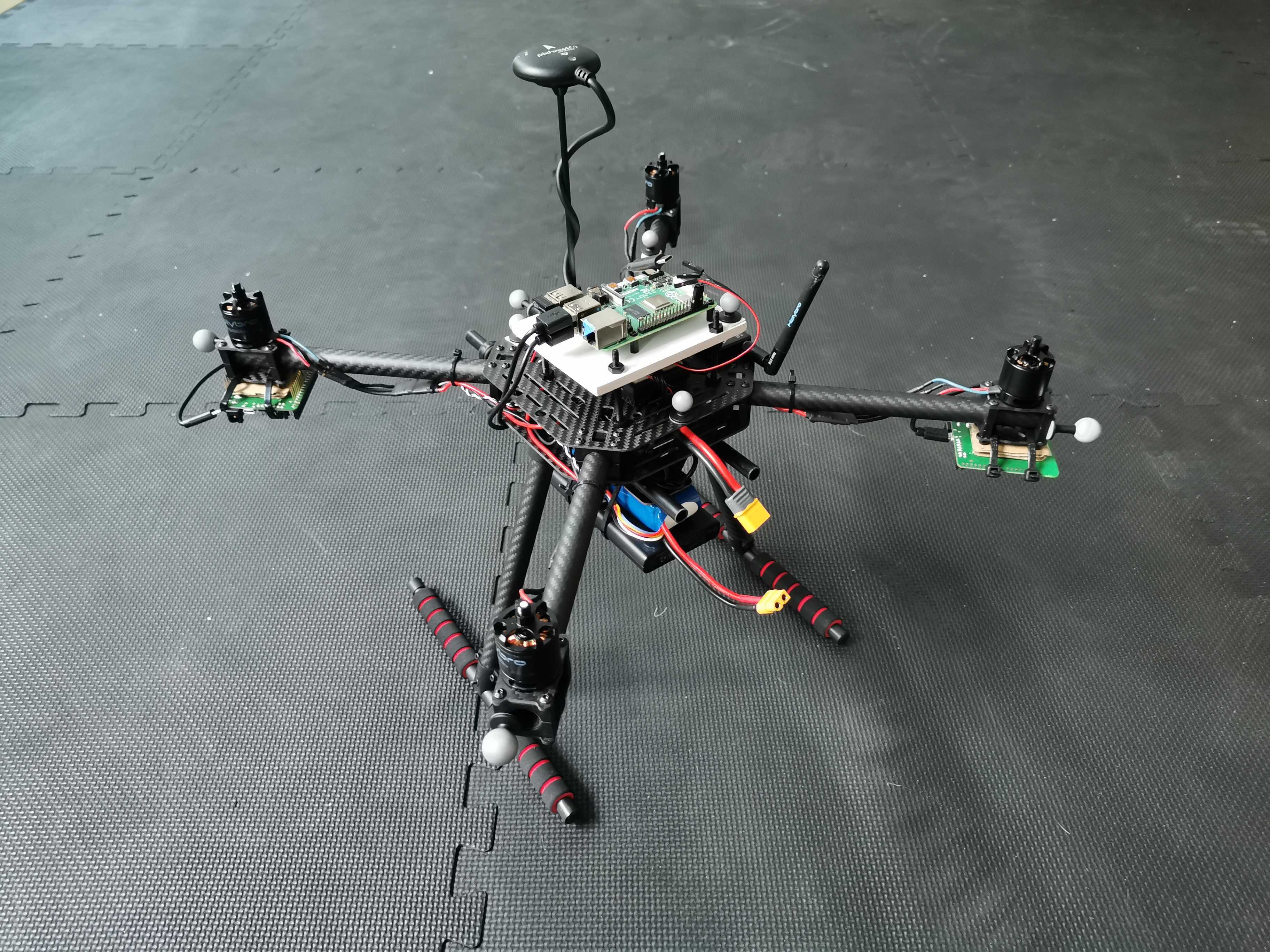}
    \vspace{-10pt}
    {\caption{(left) Pozyx UWB Developer Tag mounted to a Raspberry Pi 4B, in a case. (right) Pozyx UWB Developer Tag mounted to a quadrotor with a Pixhawk 4.}}
    \label{fig:baby_agent}
\end{figure}

Table \ref{table:exp_rmse} shows the position RMSE of the different algorithms, for each trial. The proposed method outperforms ``plain vanilla'' implementations of the SWF and the EKF by a large margin, in all of the trials, showing the value of the greedy keypoint selection strategy. Figure \ref{fig:two_agent_greedy_pos_error} shows the position estimation error of the proposed method, and that the estimate stays within the $\pm3 \sigma$ confidence bounds, except for two very brief moments. Finally, Figure \ref{fig:two_agent_greedy_pos_norm} shows the norm of the position estimation error in one of the trials. {During the periods of large error, between 20 and 40 seconds in  Figure \ref{fig:two_agent_greedy_pos_norm}, the least squares cost for all three algorithms did not dramatically increase relative to other periods. This implies that the EKF and the Vanilla SWF have converged to ambiguous states which also have low process and measurement error, again a possible consequence of unobservability.}
\begin{table}[H]
    \caption{Position RMSE of Experimental Trials}
    \vspace{-0.3cm}
    \begin{center}
    \begin{tabular}{l|c|c|c}
    \hline
    \textbf{} & \textbf{Proposed} & \textbf{Vanilla SWF}  & \textbf{EKF}\\
    \hline \hline
    Trial 1 (one static agent)& \textbf{0.68}~$\mbf{m}$ & 1.41~$\mathrm{m}$& 2.11~$\mathrm{m}$\\
    Trial 2 (both moving)& \textbf{0.55}~$\mbf{m}$&  1.22~$\mathrm{m}$  &1.71~$\mathrm{m}$\\
    Trial 3 (both moving)& \textbf{0.86}~$\mbf{m}$&  1.56~$\mathrm{m}$& 2.69~$\mathrm{m}$\\
    {Trial 4 (spinning motors)}& \textbf{0.87}~$\mbf{m}$&  1.34~$\mathrm{m}$& 1.12$~\mathrm{m}$\\
    \hline

    \end{tabular}
    \label{table:exp_rmse}
    \end{center}
    \vspace{-12pt}
\end{table}

\section{Conclusion}
This letter shows that it is possible to estimate the relative position of one UWB-equipped agent relative to another, provided that there is sufficient motion between them, and that they also possess 9-DOF IMUs. The proposed algorithm with strategically placed keypoints shows ubiquitous improvement over standard estimators, both in simulation and experiment, where there is 2 to 3-fold reduction in positioning error.  Future work should consider switching the loosely-coupled architecture for a tightly-coupled one, as the current scheme is heavily dependent on accurate attitude estimates. In environments with large magnetic perturbations, this can deteriorate performance. Furthermore, this algorithm should be tested on actual flying quadrotors, where vibrations can be more significant than those experienced in these experiments.

\appendix
\subsection{Proof of Theorem \ref{thm:observability}} \label{sec:observability_proof}
\begin{proof}
    Satisfying the rank condition in \eqref{eq:obs1} amounts to showing that there are $6$ linearly independent columns in the matrix $\mc{O}$. Since there are $K$ columns in $\mc{O}$, $K \geq 6$, which is Condition 1, is immediately required so that there are at least 6 columns. If the theorem is proven for $K=6$, then it also holds for $K >6$, as the addition of columns will not affect the linear independence of the first six. As such, using a change in notation for brevity, the rank condition in \eqref{eq:obs1} requires that
    \bdis
    \mc{O} = \bma{cccccc} \mbs{\rho}_{0} & \mbs{\rho}_{1} & \mbs{\rho}_{2} & \mbs{\rho}_{3} & \mbs{\rho}_{4} & \mbs{\rho}_{5} \\
    \mbf{0} & \Delta t_1\mbs{\rho}_{1} & \Delta t_2\mbs{\rho}_{2} & \Delta t_3\mbs{\rho}_{3} & \Delta t_4\mbs{\rho}_{4} & \Delta t_5\mbs{\rho}_{5}\ema
    \edis
    have rank 6. Without loss of generality, assume $\{ \mbs{\rho}_0,\mbs{\rho}_1,\mbs{\rho}_2\}$ are linearly independent. It immediately follows that the first three columns of $\mc{O}$ are linearly independent. Consider now $\mbs{\rho}_3$, again without loss of generality. Since $\mbs{\rho}_0,\mbs{\rho}_1,\mbs{\rho}_2 \in \mathbb{R}^3$, $\mbs{\rho}_3$ is linearly dependent on $\mbs{\rho}_0,\mbs{\rho}_1,\mbs{\rho}_2$ and can be written as 
    \beq \label{eq:obsv1}
    \mbs{\rho}_3 = \alpha_0\mbs{\rho}_0 + \alpha_1\mbs{\rho}_1 + \alpha_2\mbs{\rho}_2.
    \eeq
    If the fourth column of $\mc{O}$ were linearly dependent on the first three, then there would exist $\beta_0, \beta_1, \beta_2$ such that
    \beq \label{eq:obsv2}
    \bma{c} \mbs{\rho}_3 \\ \Delta t_3 \mbs{\rho}_3 \ema = \beta_0 \bma{c} \mbs{\rho}_0 \\ \mbf{0}\ema + \beta_1 \bma{c} \mbs{\rho}_1 \\ \Delta t_1 \mbs{\rho}_1 \ema + \beta_2 \bma{c} \mbs{\rho}_2 \\ \Delta t_2 \mbs{\rho}_2 \ema.
    \eeq
    Substituting \eqref{eq:obsv1} into \eqref{eq:obsv2} yields
    \begin{multline}\label{eq:obsv3}
    \bma{c}\alpha_0\mbs{\rho}_0 + \alpha_1\mbs{\rho}_1 + \alpha_2\mbs{\rho}_2 \\ \Delta t_3 (\alpha_0\mbs{\rho}_0 + \alpha_1\mbs{\rho}_1 + \alpha_2\mbs{\rho}_2) \ema = \beta_0 \bma{c} \mbs{\rho}_0 \\ \mbf{0}\ema \\+ \beta_1 \bma{c} \mbs{\rho}_1 \\ \Delta t_1 \mbs{\rho}_1 \ema + \beta_2 \bma{c} \mbs{\rho}_2 \\ \Delta t_2 \mbs{\rho}_2 \ema.
    \end{multline}
    The first three lines of \eqref{eq:obsv3} immediately require that $\alpha_0 = \beta_0$, $\alpha_1 = \beta_1,$ $\alpha_2 = \beta_2$, which reduces the last three lines of \eqref{eq:obsv3} to 
    \beq
    \Delta t_3 \alpha_0\mbs{\rho}_0 + \Delta t_3\alpha_1\mbs{\rho}_1 + \Delta t_3\alpha_2\mbs{\rho}_2
    = \alpha_1 \Delta t_1 \mbs{\rho}_1 + \alpha_2 \Delta t_2 \mbs{\rho}_2.
    \eeq
    Hence, the fourth column of $\mc{O}$ is linearly dependent on the first three if and only if $\alpha_0 = 0,$ $ \Delta t_3 = \Delta t_2,$ $ \Delta t_3 = \Delta t_1$. Condition 2 in Theorem \ref{thm:observability} means $ \Delta t_3 \neq \Delta t_2 \neq \Delta t_1$, hence proving the linear independence of the fourth column of $\mc{O}$ from the first three. Identical proofs show that the fifth and sixth columns of $\mc{O}$ are also linearly independent from the first three. If additionally $\{ \mbs{\rho}_3,\mbs{\rho}_4,\mbs{\rho}_5\}$ are linearly independent, then the fourth, fifth, and sixth columns of $\mc{O}$ will be linearly independent from each other, in addition to each being linearly independent from the first three. Hence, all 6 columns of $\mc{O}$ will be linearly independent, giving it a rank of 6. Since this proof was done for a completely arbitrary choice of linearly independent sets $\{ \mbs{\rho}_0,\mbs{\rho}_1,\mbs{\rho}_2\}$ and $\{ \mbs{\rho}_3,\mbs{\rho}_4,\mbs{\rho}_5\}$, it is only required that there exists two distinct non-intersecting subsets of $\{ \mbs{\rho}_0,\mbs{\rho}_1,\ldots,\mbs{\rho}_K\}$, each containing 3 linearly independent elements, which is Condition 3.
\end{proof}

{\AtNextBibliography{\small}
\printbibliography}

\end{document}